\pdfoutput=1
\documentclass[english]{article}
\usepackage[T1]{fontenc}
\usepackage[latin9]{inputenc}
\usepackage{array}
\usepackage{float}
\usepackage{url}
\usepackage{multirow}
\usepackage{amsmath}
\usepackage{amsthm}
\usepackage{amssymb}
\usepackage{stackrel}
\usepackage{graphicx}
 \usepackage[numbers]{natbib}
\makeatletter

\providecommand{\tabularnewline}{\\}
\floatstyle{ruled}
\newfloat{algorithm}{tbp}{loa}
\providecommand{\algorithmname}{Algorithm}
\floatname{algorithm}{\protect\algorithmname}

\theoremstyle{plain}
\newtheorem{thm}{\protect\theoremname}
\theoremstyle{plain}
\newtheorem{lem}[thm]{\protect\lemmaname}

\usepackage[preprint]{nips_2018}
\usepackage{algorithm,algpseudocode}
\usepackage{amsmath}
\algnewcommand\algorithmicforeach{\textbf{for each}}
\algdef{S}[FOR]{ForEach}[1]{\algorithmicforeach\ #1\ \algorithmicdo}

\makeatother

\usepackage{babel}
\providecommand{\lemmaname}{Lemma}
\providecommand{\theoremname}{Theorem}

\begin{document}

\title{Variational Memory Encoder-Decoder}

\author{Hung Le, Truyen Tran, Thin Nguyen and Svetha Venkatesh\\
Applied AI Institute, Deakin University, Geelong, Australia\\
\texttt{\{lethai,truyen.tran,thin.nguyen,svetha.venkatesh\}@deakin.edu.au}}
\maketitle
\begin{abstract}
Introducing variability while maintaining coherence is a core task
in learning to generate utterances in conversation. Standard neural
encoder-decoder models and their extensions using conditional variational
autoencoder often result in either trivial or digressive responses.
To overcome this, we explore a novel approach that injects variability
into neural encoder-decoder via the use of external memory as a mixture
model, namely Variational Memory Encoder-Decoder (VMED). By associating
each memory read with a mode in the latent mixture distribution at
each timestep, our model can capture the variability observed in sequential
data such as natural conversations. We empirically compare the proposed
model against other recent approaches on various conversational datasets.
The results show that VMED consistently achieves significant improvement
over others in both metric-based and qualitative evaluations. 
 
\end{abstract}

\section{Introduction}

Recent advances in generative modeling have led to exploration of
generative tasks. While generative models such as GAN \cite{goodfellow2014generative}
and VAE \cite{king14autoencoder,rezende2014stochastic} have been
applied successfully for image generation, learning generative models
for sequential discrete data is a long-standing problem. Early attempts
to generate sequences using RNNs \cite{graves2013generating} and
neural encoder-decoder models \cite{kalchbrenner2013recurrent,vinyals2015neural}
gave promising results, but the deterministic nature of these models
proves to be inadequate in many realistic settings. Tasks such as
translation, question-answering and dialog generation would benefit
from stochastic models that can produce a variety of outputs for an
input. For example, there are several ways to translate a sentence
from one language to another, multiple answers to a question and multiple
responses for an utterance in conversation. 

Another line of research that has captured attention recently is memory
augmented neural networks (MANNs). Such models have larger memory
capacity and thus ``remember'' temporally distant information in
the input sequence and provide a RAM-like mechanism to support model
execution. MANNs have been successfully applied to long sequence prediction
tasks \cite{graves2016hybrid,NIPS2015_5846} demonstrating great improvement
when compared to other recurrent models. However, the role of memory
in sequence generation has not been well understood. 

For tasks involving language understanding and production, handling
intrinsic uncertainty and latent variations is necessary. The choice
of words and grammars may change erratically depending on speaker
intentions, moods and previous languages used. The underlying RNN
in neural sequential models finds it hard to capture the dynamics
and their outputs are often trivial or too generic \cite{li2016diversity}.
One way to overcome these problems is to introduce variability into
these models. Unfortunately, sequential data such as speech and natural
language is a hard place to inject variability \cite{serban2017hierarchical}
since they require a coherence of grammars and semantics yet allow
freedom of word choice.

We propose a novel hybrid approach that integrates MANN and VAE, called
Variational Memory Encoder-Decoder (VMED), to model the sequential
properties and inject variability in sequence generation tasks. We
introduce latent random variables to model the variability observed
in the data and capture dependencies between the latent variables
across timesteps. Our assumption is that there are latent variables
governing an output at each timestep. In the conversation context,
for instance, the latent space may represent the speaker's hidden
intention and mood that dictate word choice and grammars. For a rich
latent multimodal space, we use a Mixture of Gaussians (MoG) because
a spoken word's latent intention and mood can come from different
modes, e.g., whether the speaker is asking or answering, or she/he
is happy or sad. By modeling the latent space as an MoG where each
mode associates with some memory slot, we aim to capture multiple
modes of the speaker's intention and mood when producing a word in
the response. Since the decoder in our model has multiple read heads,
the MoG can be computed directly from the content of chosen memory
slots. Our external memory plays a role as a mixture model distribution
generating the latent variables that are used to produce the output
and take part in updating the memory for future generative steps. 

To train our model, we adapt Stochastic Gradient Variational Bayes
(SGVB) framework \cite{king14autoencoder}. Instead of minimizing
the $KL$ divergence directly, we resort to using its variational
approximation \cite{hershey2007approximating} to accommodate the
MoG in the latent space. We show that minimizing the approximation
results in $KL$ divergence minimization. We further derive an upper
bound on our total timestep-wise $KL$ divergence and demonstrate
that minimizing the upper bound is equivalent to fitting a continuous
function by a scaled MoG. We validate the proposed model on the task
of conversational response generation. This task serves as a nice
testbed for the model because an utterance in a conversation is conditioned
on previous utterances, the intention and the mood of the speaker.
Finally, we evaluate our model on two open-domain and two closed-domain
conversational datasets. The results demonstrate our proposed VMED
gains significant improvement over state-of-the-art alternatives.

\section{Preliminaries}

\subsection{Memory-augmented Encoder-Decoder Architecture}

A memory-augmented encoder-decoder (MAED) consists of two neural controllers
linked via external memory. This is a natural extension to read-write
MANNs to handle sequence-to-sequence problems. In MAED, the memory
serves as a compressor that encodes the input sequence to its memory
slots, capturing the most essential information. Then, a decoder will
attend to these memory slots looking for the cues that help to predict
the output sequence. MAED has recently demonstrated promising results
in machine translation \cite{britz2017efficient,wang2016memory} and
healthcare \cite{le2018dual_con,Le:2018:DMN:3219819.3219981,prakash2017condensed}.
In this paper, we advance a recent MAED known as DC-MANN described
in \cite{le2018dual_con} where the powerful DNC \cite{graves2016hybrid}
is chosen as the external memory. In DNC, memory accesses and updates
are executed via the controller's reading and writing heads at each
timestep. Given current input $x_{t}$ and a set of $K$ previous
read values from memory $r_{t-1}=\left[r_{t-1}^{1},r_{t-1}^{2},...,r_{t-1}^{K}\right]$,
the controllers compute read-weight vector $w_{t}^{i,r}$ and write-weight
vector $w_{t}^{w}$ for addressing the memory $M_{t}$. There are
two versions of decoding in DC-MANN: write-protected and writable
memory. We prefer to allow writing to the memory during inference
because in this work, we focus on generating diverse output sequences,
which requires a dynamic memory for both encoding and decoding process. 

\subsection{Conditional Variational Autoencoder (CVAE) for Conversation Generation}

A dyadic conversation can be represented via three random variables:
the conversation context $x$ (all the chat before the response utterance),
the response utterance $y$ and a latent variable $z$, which is used
to capture the latent distribution over the reasonable responses.
A variational autoencoder conditioned on $x$ (CVAE) is trained to
maximize the conditional log likelihood of $y$ given $x$, which
involves an intractable marginalization over the latent variable $z$,
i.e.,:

\begin{equation}
p\left(y\mid x\right)=\int_{z}p\left(y,z\mid x\right)dz=\int_{z}p\left(y\mid x,z\right)p\left(z\mid x\right)dz
\end{equation}
Fortunately, CVAE can be efficiently trained with the Stochastic Gradient
Variational Bayes (SGVB) framework \cite{king14autoencoder} by maximizing
the variational lower bound of the conditional log likelihood. In
a typical CVAE work, $z$ is assumed to follow multivariate Gaussian
distribution with a diagonal covariance matrix, which is conditioned
on $x$ as $p_{\phi}\left(z\mid x\right)$ and a recognition network
$q_{\theta}(z\mid x,y)$ to approximate the true posterior distribution
$p(z\mid x,y).$ The variational lower bound becomes:

\begin{alignat}{1}
L\left(\phi,\theta;y,x\right)= & -KL\left(q_{\theta}\left(z\mid x,y\right)\parallel p_{\phi}\left(z\mid x\right)\right)+\mathbb{E}_{q_{\theta}\left(z\mid x,y\right)}\left[\log p\left(y\mid x,z\right)\right]\leq\log p\left(y\mid x\right)\label{eq:L_CVAE}
\end{alignat}
With the introduction of the neural approximator $q_{\theta}(z\mid x,y)$
and the reparameterization trick \cite{kingma2014semi}, we can apply
the standard back-propagation to compute the gradient of the variational
lower bound. Fig. \ref{fig:Graphical-Model-of}(a) depicts elements
of the graphical model for this approach in the case of using CVAE.

\section{Methods}

Built upon CVAE and partly inspired by VRNN \cite{chung2015recurrent},
we introduce a novel memory-augmented variational recurrent network
dubbed Variational Memory Encoder-Decoder (VMED). With an external
memory module, VMED explicitly models the dependencies between latent
random variables across subsequent timesteps. However, unlike the
VRNN which uses hidden values of RNN to model the latent distribution
as a Gaussian, our VMED uses read values $r$ from an external memory
$M$ as a Mixture of Gaussians (MoG) to model the latent space. This
choice of MoG also leads to new formulation for the prior $p_{\phi}$
and the posterior $q_{\theta}$ mentioned in Eq. (\ref{eq:L_CVAE}).
The graphical representation of our model is shown in Fig. \ref{fig:Graphical-Model-of}(b).

\begin{figure}
\begin{centering}
\includegraphics[width=1\linewidth]{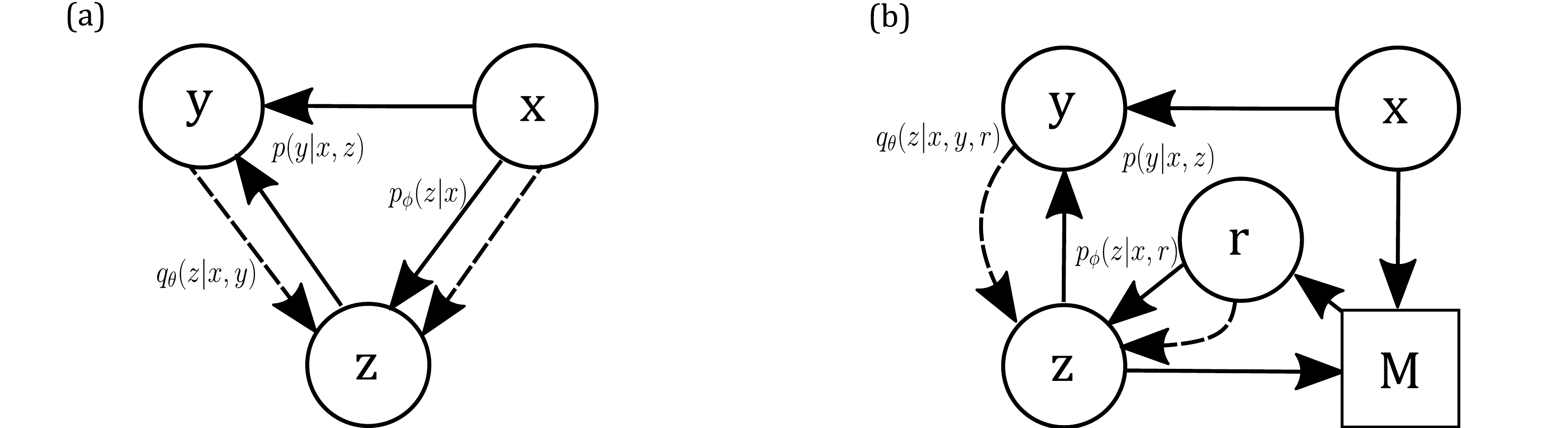}\caption{Graphical Models of the vanilla CVAE (a) and our proposed VMED (b)\label{fig:Graphical-Model-of}}
\par\end{centering}
\end{figure}

\subsection{Generative Process}

The VMED includes a CVAE at each time step of the decoder. These CVAEs
are conditioned on the context sequence via $K$ read values $r_{t-1}=\left[r_{t-1}^{1},r_{t-1}^{2},...,r_{t-1}^{K}\right]$
from the external memory. Since the read values are conditioned on
the previous state of the decoder $h_{t-1}^{d}$, our model takes
into account the temporal structure of the output. Unlike other designs
of CVAE where there is often only one CVAE with a Gaussian prior for
the whole decoding process, our model keeps reading the external memory
to produce the prior as a Mixture of Gaussians at every timestep.
At the $t$-th step of generating an utterance in the output sequence,
the decoder will read from the memory $K$ read values, representing
$K$ modes of the MoG. This multi-modal prior reflects the fact that
given a context $x$, there are different modes of uttering the output
word $y_{t}$, which a single mode cannot fully capture. The MoG prior
distribution is modeled as:

\begin{equation}
g_{t}=p_{\phi}\left(z_{t}\mid x,r_{t-1}\right)=\stackrel[i=1]{K}{\sum}\pi_{t}^{i,x}\left(x,r_{t-1}^{i}\right)\mathcal{N}\left(z_{t};\mu_{t}^{i,x}\left(x,r_{t-1}^{i}\right),\sigma_{t}^{i,x}\left(x,r_{t-1}^{i}\right)^{2}\mathbf{I}\right)\label{eq:mog}
\end{equation}
We treat the mean $\mu_{t}^{i,x}$ and standard deviation (s.d.) $\sigma_{t}^{i,x}$
of each Gaussian distribution in the prior as neural functions of
the context sequence $x$ and read vectors from the memory. The context
is encoded into the memory by an $LSTM^{E}$ encoder. In decoding,
the decoder $LSTM^{D}$ attends to the memory and choose $K$ read
vectors. We split each read vector into two parts $r^{i,\mu}$ and
$r^{i,\sigma}$ , each of which is used to compute the mean and s.d.,
respectively: $\mu_{t}^{i,x}=r_{t-1}^{i,\mu}$, $\sigma_{t}^{i,x}=softplus\left(r_{t-1}^{i,\sigma}\right)$.
Here we use the softplus function for computing s.d. to ensure the
positiveness. The mode weight $\pi_{t}^{i,x}$ is chosen based on
the read attention weights $w_{t-1}^{i,r}$ over memory slots. Since
we use soft-attention, a read value is computed from all slots yet
the main contribution comes from the one with highest attention score.
Thus, we pick the maximum attention score in each read weight and
normalize to become the mode weights: $\pi_{t}^{i,x}=\max\,w_{t-1}^{i,r}/\stackrel[i=1]{i=K}{\sum}\max\,w_{t-1}^{i,r}$. 

Armed with the prior, we follow a recurrent generative process by
alternatively using the memory to compute the MoG and using latent
variable $z$ sampled from the MoG to update the memory and produce
the output conditional distribution. The pseudo-algorithm of the generative
process is given in Algorithm \ref{alg:Generation-step-of}. 

\begin{algorithm}[t]
\begin{algorithmic}[1]
\small
\Require{Given $p_{\phi}$, $\left[r_{0}^{1},r_{0}^{2},...,r_{0}^{K}\right]$, $h_0^d$, $y_{0}^{*}$}
\For{$t=1,T$}
\State{Sampling $z_{t}\sim p_{\phi}\left(z_{t}\mid x,r_{t-1}\right)$ in Eq. (3)}
\State{Compute: $o_{t}^{d},h_{t}^{d}=LSTM^{D}\left(\left[y_{t-1}^{*},z_{t}\right],h_{t-1}^{d}\right)$}
\State{Compute the conditional distribution: $p\left(y_{t}\mid x,z_{\leq t}\right)=softmax\left(W_{out}o_{t}^{d}\right)$}
\State{Update memory and read $[r_{t}^{1},r_{t}^{2},...,r_{t}^{K}]$ using $h_t^d$ as in DNC}
\State{Generate output $y_{t}^{*}=\underset{y\in Vocab}{argmax}\,p\left(y_{t}=y\mid x,z_{\leq t}\right)$}
\EndFor
\end{algorithmic} 

\caption{VMED Generation\label{alg:Generation-step-of}}

\end{algorithm}

\subsection{Neural Posterior Approximation}

At each step of the decoder, the true posterior $p\left(z_{t}\mid x,y\right)$
will be approximated by a neural function of $x,y$ and $r_{t-1}$,
denoted as $q_{\theta}\left(z_{t}\mid x,y,r_{t-1}\right)$ . Here,
we use a Gaussian distribution to approximate the posterior. The unimodal
posterior is chosen because given a response $y$, it is reasonable
to assume only one mode of latent space is responsible for this response.
Also, choosing a unimodel will allow the reparameterization trick
during training and reduce the complexity of $KL$ divergence computation.
The approximated posterior is computed by the following the equation:

\begin{equation}
f_{t}=q_{\theta}\left(z_{t}\mid x,y_{\leq t},r_{t-1}\right)=\mathcal{N}\left(z_{t};\mu_{t}^{x,y}\left(x,y_{\leq t},r_{t-1}\right),\sigma_{t}^{x,y}\left(x,y_{\leq t},r_{t-1}\right)^{2}\textrm{\ensuremath{\mathbf{I}}}\right)\label{eq:q_theta}
\end{equation}
 with mean $\mu_{t}^{x,y}$ and s.d. $\sigma_{t}^{x,y}$. We use
an $LSTM^{U}$ utterance encoder to model the ground truth utterance
sequence up to timestep $t$-th $y_{\leq t}$. The $t$-th hidden
value of the $LSTM^{U}$ is used to represent the given data in the
posterior: $h_{t}^{u}=LSTM^{U}\left(y_{t},h_{t-1}^{u}\right)$. The
neural posterior combines the read values $\mathbf{r}_{t}=\stackrel[i=1]{K}{\sum}\pi_{t}^{i,x}r_{t-1}^{i}$
together with the ground truth data to produce the Gaussian posterior:
$\mu_{t}^{x,y}=W_{\mu}\left[\mathbf{r}_{t},h_{t}^{u}\right]$, $\sigma_{t}^{x,y}=softplus\left(W_{\sigma}\left[\mathbf{r}_{t},h_{t}^{u}\right]\right)$.
In these equations, we use learnable matrix weights $W_{\mu}$ and
$W_{\sigma}$ as a recognition network to compute the mean and s.d.
of the posterior, ensuring that the distribution has the same dimension
as the prior. We apply the reparamterization trick to calculate the
random variable sampled from the posterior as $z'_{t}=\mu_{t}^{x,y}+\sigma_{t}^{x,y}\odot\epsilon$,
$\epsilon\in\mathcal{N}\left(0,\mathbf{I}\right)$. Intuitively, the
reparameterization trick bridges the gap between the generation model
and the inference model during the training. 

\subsection{Learning}

In the training phase, the neural posterior is used to produce the
latent variable $z'_{t}$. The read values from memory are used directly
as the MoG priors and the priors are trained to approximate the posterior
by reducing the $KL$ divergence. During testing, the decoder uses
the prior for generating latent variable $z_{t}$, from which the
output is computed. The training and testing diagram is illustrated
in Fig. \ref{fig:Training-of-VM3NN}. The objective function becomes
a timestep-wise variational lower bound by following similar derivation
presented in \cite{chung2015recurrent}:

\begin{equation}
\mathcal{L}\left(\theta,\phi;y,x\right)=E_{q*}\left[\stackrel[t=1]{T}{\sum}-KL\left(q_{\theta}\left(z_{t}\mid x,y_{\leq t},r_{t-1}\right)\parallel p_{\phi}\left(z_{t}\mid x,r_{t-1}\right)\right)+\log p\left(y_{t}\mid x,z_{\leq t}\right)\right]\label{eq:exloss}
\end{equation}
where $q*=q_{\theta}\left(z_{\leq T}\mid x,y_{\leq T},r_{<T}\right)$.
To maximize the objective function, we have to compute $KL$ divergence
between $f_{t}=q_{\theta}\left(z_{t}\mid x,y_{\leq t},r_{t-1}\right)$
and $g_{t}=p_{\phi}\left(z_{t}\mid x,r_{t-1}\right)$. Since there
is no closed-form for this $KL\left(f_{t}\parallel g_{t}\right)$
between Gaussian $f_{t}$ and Mixture of Gaussians $g_{t}$, we use
a closed-form approximation named $D_{var}$ \cite{hershey2007approximating}
to replace the $KL$ term in the objective function. For our case:
$KL\left(f_{t}\parallel g_{t}\right)\approx D_{var}\left(f_{t}\parallel g_{t}\right)=-\log\stackrel[i=1]{K}{\sum}\pi^{i}e^{-KL\left(f_{t}\parallel g_{t}^{i}\right)}$.
Here, $KL\left(f_{t}\parallel g_{t}^{i}\right)$ is the $KL$ divergence
between two Gaussians and $\pi^{i}$ is the mode weight of $g_{t}$.
The final objective function is:

\begin{equation}
\begin{alignedat}{1}\mathcal{L}= & \stackrel[t=1]{T}{\sum}\log\stackrel[i=1]{K}{\sum}\left[\pi_{t}^{i,x}\exp\left(-KL\left(\mathcal{N}\left(\mu_{t}^{x,y},\sigma_{t}^{x,y}{}^{2}\textrm{\ensuremath{\mathbf{I}}}\right)\parallel\mathcal{N}\left(\mu_{t}^{i,x},\sigma_{t}^{i,x}{}^{2}\mathbf{I}\right)\right)\right)\right]\\
 & +\frac{1}{L}\stackrel[t=1]{T}{\sum}\stackrel[l=1]{L}{\sum}\log p\left(y_{t}\mid x,z_{\leq t}^{(l)}\right)
\end{alignedat}
\label{eq:loss}
\end{equation}

\begin{figure}
\begin{centering}
\includegraphics[width=0.7\linewidth]{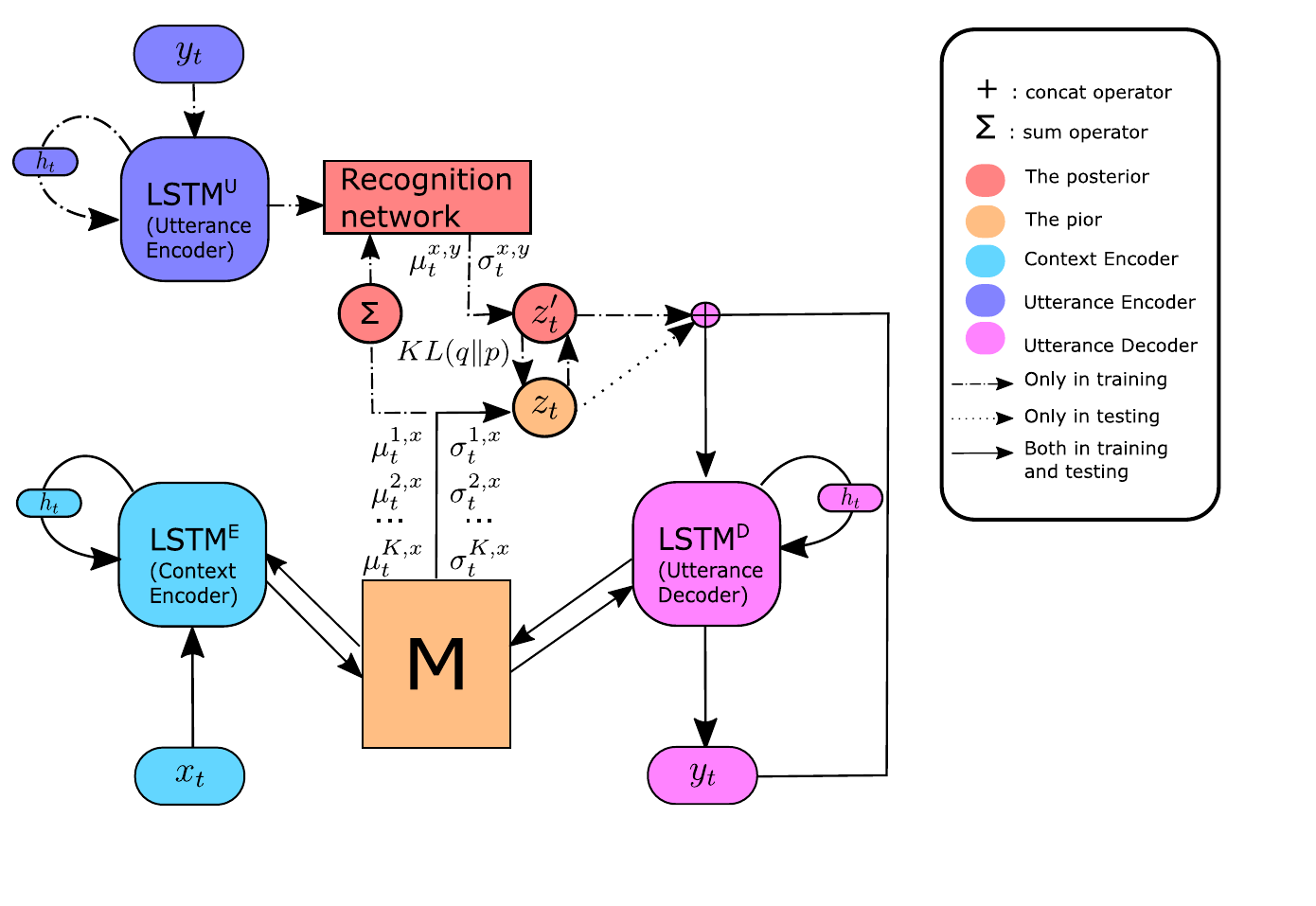}
\par\end{centering}
\caption{Training and testing of VMED\label{fig:Training-of-VM3NN}}
\end{figure}

\subsection{Theoretical Analysis }

We now show that by modeling the prior as MoG and the posterior as
Gaussian, minimizing the approximation results in $KL$ divergence
minimization. Let define the log-likelihood $L_{f}\left(g\right)=E_{f\left(x\right)}\left[\log g\left(x\right)\right]$,
we have (see Supplementary material for full derivation):

\begin{alignat*}{1}
L_{f}\left(g\right)\geq & \log\stackrel[i=1]{K}{\sum}\pi^{i}e^{-KL\left(f\parallel g^{i}\right)}+L_{f}\left(f\right)=-D_{var}+L_{f}\left(f\right)\\
\Rightarrow D_{var}\geq & L_{f}\left(f\right)-L_{f}\left(g\right)=KL\left(f\parallel g\right)
\end{alignat*}
Thus, minimizing $D_{var}$ results in $KL$ divergence minimization.
Next, we establish an upper bound on the total timestep-wise $KL$
divergence in Eq. (\ref{eq:exloss}) and show that minimizing this
upper bound is equivalent to fitting a continuous function by a scaled
MoG. The total timestep-wise $KL$ divergence reads:

\begin{alignat*}{1}
\stackrel[t=1]{T}{\sum}KL\left(f_{t}\parallel g_{t}\right)= & \stackrel[-\infty]{+\infty}{\int}\stackrel[t=1]{T}{\sum}f_{t}\left(x\right)\log\left[f_{t}\left(x\right)\right]dx\quad-\stackrel[-\infty]{+\infty}{\int}\stackrel[t=1]{T}{\sum}f_{t}\left(x\right)\log\left[g_{t}\left(x\right)\right]dx
\end{alignat*}
where $g_{t}=\stackrel[i=1]{K}{\sum}\pi_{t}^{i}g_{t}^{i}$ and $g_{t}^{i}$
is the $i$-th Gaussian in the MoG at timestep $t$-th. If at each
decoding step, minimizing $D_{var}$ results in adequate $KL$ divergence
such that the prior is optimized close to the neural posterior, according
to Chebyshev's sum inequality, we can derive an upper bound on the
total timestep-wise $KL$ divergence as (see  Supplementary Materials
 for full derivation):

\begin{equation}
\stackrel[-\infty]{+\infty}{\int}\stackrel[t=1]{T}{\sum}f_{t}\left(x\right)\log\left[f_{t}\left(x\right)\right]dx\quad-\stackrel[-\infty]{+\infty}{\int}\frac{1}{T}\stackrel[t=1]{T}{\sum}f_{t}\left(x\right)\log\left[\stackrel[t=1]{T}{\prod}g_{t}\left(x\right)\right]dx\label{eq:fupper}
\end{equation}
The left term is sum of the entropies of $f_{t}\left(x\right)$, which
does not depend on the training parameter $\phi$ used to compute
$g_{t}$, so we can ignore that. Thus given $f$, minimizing the upper
bound of the total timestep-wise $KL$ divergence is equivalent to
maximizing the right term of Eq. (\ref{eq:fupper}). Since $g_{t}$
is an MoG and products of MoG is proportional to an MoG, $\stackrel[t=1]{T}{\prod}g_{t}\left(x\right)$
is a scaled MoG (see  Supplementary material  for full proof). Maximizing
the right term is equivalent to fitting function $\stackrel[t=1]{T}{\sum}f_{t}\left(x\right)$,
which is sum of Gaussians and thus continuous, by a scaled MoG. This,
in theory, is possible regardless of the form of $f_{t}$ since MoG
is a universal approximator \cite{bacharoglou2010approximation,maz1996approximate}.

\section{Results}

\begin{table}
\begin{centering}
\caption{BLEU-1, 4 and A-Glove on testing datasets. B1, B4, AG are acronyms
for BLEU-1, BLEU-4, A-Glove metrics, respectively (higher is better).
\label{tab:BLEUs-and-A-Glove}}
~
\par\end{centering}
\centering{}{\small{}}%
\begin{tabular}{|c|c|c|c|c|c|c|c|c|c|c|c|c|}
\hline 
\multirow{2}{*}{{\small{}Model}} & \multicolumn{3}{c|}{{\small{}Cornell Movies}} & \multicolumn{3}{c|}{{\small{}OpenSubtitle}} & \multicolumn{3}{c|}{{\small{}LJ users}} & \multicolumn{3}{c|}{{\small{}Reddit comments}}\tabularnewline
\cline{2-13} 
 & {\small{}B1} & {\small{}B4} & {\small{}AG} & {\small{}B1} & {\small{}B4} & {\small{}AG} & {\small{}B1} & {\small{}B4} & {\small{}AG} & {\small{}B1} & {\small{}B4} & {\small{}AG}\tabularnewline
\hline 
{\small{}Seq2Seq} & {\small{}18.4} & {\small{}9.5} & {\small{}0.52} & {\small{}11.4} & {\small{}5.4} & {\small{}0.29} & {\small{}13.1} & {\small{}6.4} & {\small{}0.45} & {\small{}7.5} & {\small{}3.3} & {\small{}0.31}\tabularnewline
\hline 
{\small{}Seq2Seq-att} & {\small{}17.7} & {\small{}9.2} & {\small{}0.54} & {\small{}13.2} & {\small{}6.5} & {\small{}0.42} & {\small{}11.4} & {\small{}5.6} & {\small{}0.49} & {\small{}5.5} & {\small{}2.4} & {\small{}0.25}\tabularnewline
\hline 
{\small{}DNC} & {\small{}17.6} & {\small{}9.0} & {\small{}0.51} & {\small{}14.3} & {\small{}7.2} & {\small{}0.47} & {\small{}12.4} & {\small{}6.1} & {\small{}0.47} & {\small{}7.5} & {\small{}3.4} & {\small{}0.28}\tabularnewline
\hline 
{\small{}CVAE} & {\small{}16.5} & {\small{}8.5} & {\small{}0.56} & {\small{}13.5} & {\small{}6.6} & {\small{}0.45} & {\small{}12.2} & {\small{}6.0} & {\small{}0.48} & {\small{}5.3} & {\small{}2.8} & {\small{}0.39}\tabularnewline
\hline 
{\small{}VLSTM} & {\small{}18.6} & {\small{}9.7} & {\small{}0.59} & {\small{}16.4} & {\small{}8.1} & {\small{}0.43} & {\small{}11.5} & {\small{}5.6} & {\small{}0.46} & {\small{}6.9} & {\small{}3.1} & {\small{}0.27}\tabularnewline
\hline 
\hline 
{\small{}VMED (K=1)} & {\small{}20.7} & {\small{}10.8} & {\small{}0.57} & {\small{}12.9} & {\small{}6.2} & {\small{}0.44} & {\small{}13.7} & {\small{}6.9} & {\small{}0.47} & {\small{}9.1} & {\small{}4.3} & {\small{}0.39}\tabularnewline
\hline 
{\small{}VMED (K=2)} & {\small{}22.3} & {\small{}11.9} & \textbf{\small{}0.64} & {\small{}15.3} & {\small{}8.8} & {\small{}0.49} & {\small{}15.4} & {\small{}7.9} & \textbf{\small{}0.51} & {\small{}9.2} & {\small{}4.4} & {\small{}0.38}\tabularnewline
\hline 
{\small{}VMED (K=3)} & {\small{}19.4} & {\small{}10.4} & {\small{}0.63} & \textbf{\small{}24.8} & \textbf{\small{}12.9} & \textbf{\small{}0.54} & \textbf{\small{}18.1} & \textbf{\small{}9.8} & {\small{}0.49} & \textbf{\small{}12.3} & \textbf{\small{}6.4} & \textbf{\small{}0.46}\tabularnewline
\hline 
{\small{}VMED (K=4)} & \textbf{\small{}23.1} & \textbf{\small{}12.3} & {\small{}0.61} & {\small{}17.9} & {\small{}9.3} & {\small{}0.52} & {\small{}14.4} & {\small{}7.5} & {\small{}0.47} & {\small{}8.6} & {\small{}4.6} & {\small{}0.41}\tabularnewline
\hline 
\end{tabular}{\small\par}
\end{table}

\textbf{Datasets and pre-processing}: We perform experiments on two
collections: The first collection includes open-domain movie transcript
datasets containing casual conversations: Cornell Movies\footnote{\url{http://www.cs.cornell.edu/~cristian/Cornell_Movie-Dialogs_Corpus.html}}
and OpenSubtitle\footnote{\url{http://opus.nlpl.eu/OpenSubtitles.php}}.
They have been used commonly in evaluating conversational agents \cite{lison2017not,vinyals2015neural}.
The second are closed-domain datasets crawled from specific domains,
which are question-answering of LiveJournal (LJ) users and Reddit
comments on movie topics. For each dataset, we use 10,000 conversations
for validating and 10,000 for testing. 

\textbf{Baselines, implementations and metrics}: We compare our model
with three deterministic baselines: the encoder-decoder neural conversational
model (Seq2Seq) similar to \cite{vinyals2015neural} and its two variants
equipped with attention mechanism \cite{bahdanau2015neural} (Seq2Seq-att)
and a DNC external memory \cite{graves2016hybrid} (DNC). The vanilla
CVAE is also included in the baselines. To build this CVAE, we follow
similar architecture introduced in \cite{zhao2017learning} without
bag-of-word loss and dialog act features\footnote{Another variant of non-memory CVAE with MoG prior is also examined.
We produce a set of MoG parameters by a feed forward network with
the input as the last encoder hidden states. However, the model is
hard to train and fails to converge with these datasets.}. A variational recurrent model without memory is also included in
the baselines. The model termed VLSTM is implemented based on LSTM
instead of RNN as in VRNN framework \cite{chung2015recurrent}. We
try our model VMED\footnote{Source code is available at \url{https://github.com/thaihungle/VMED}}
with different number of modes ($K=1,2,3,4$). It should be noted
that, when $K=1$, our model's prior is exactly a Gaussian and the
$KL$ term in Eq. (\ref{eq:loss}) is no more an approximation. Details
of dataset descriptions and model implementations are included in
 Supplementary material.

We report results using two performance metrics in order to evaluate
the system from various linguistic points of view: (i) Smoothed Sentence-level
BLEU \cite{chen2014systematic}: BLEU is a popular metric that measures
the geometric mean of modified ngram precision with a length penalty.
We use BLEU-1 to 4 as our lexical similarity. (ii) Cosine Similarly
of Sentence Embedding: a simple method to obtain sentence embedding
is to take the average of all the word embeddings in the sentences
\cite{forgues2014bootstrapping}. We follow \cite{zhao2017learning}
and choose Glove \cite{levy2014neural} as the word embedding in measuring
sentence similarly (A-Glove). To measure stochastic models, for each
input, we generate output ten times. The metric between the ground
truth and the generated output is calculated and taken average over
ten responses. 

\begin{table}
\begin{centering}
\caption{Examples of context-response pairs. /{*}/ denotes separations between
stochastic responses.\label{tab:Examples-of-context-response}}
~
\par\end{centering}
\centering{}%
\begin{tabular}{>{\raggedright}p{0.2\columnwidth}|>{\raggedright}p{0.73\columnwidth}}
\hline 
{\footnotesize{}Input context} & {\footnotesize{}Response}\tabularnewline
\hline 
\multirow{6}{0.2\columnwidth}{\textbf{\footnotesize{}Reddit comment:}{\footnotesize{} What is your
favorite scene in film history ? Mine is the restaurant scene in the
Godfather. }} & \multirow{1}{0.73\columnwidth}{\textbf{\footnotesize{}Seq2Seq:}{\footnotesize{} The scene in}}\tabularnewline
 & \textbf{\footnotesize{}Seq2Seq-att: }{\footnotesize{}The final}\tabularnewline
 & \textbf{\footnotesize{}DNC: }{\footnotesize{}The scene in}\tabularnewline
 & \textbf{\footnotesize{}CVAE:}{\footnotesize{} Inception god! Not by
a shark /{*}/ Amour great /{*}/ Pro thing you know 3 dead}\tabularnewline
 & \textbf{\footnotesize{}VLSTM: }{\footnotesize{}The scene in /{*}/
The of a dead /{*}/ The sky in scene}\tabularnewline
 & \textbf{\footnotesize{}VMED (K=3): }{\footnotesize{}The opening scene
from history movie /{*}/ The scene in a shot nights! Robin movie /{*}/
The psycho scene in fight from }\tabularnewline
\hline 
\multirow{6}{0.2\columnwidth}{\textbf{\footnotesize{}Reddit comment:}{\footnotesize{} What actor
will win an Oscar in the next 10 years ?}} & \textbf{\footnotesize{}Seq2Seq: }{\footnotesize{}Colin}\tabularnewline
 & \textbf{\footnotesize{}Seq2Seq-att: }{\footnotesize{}Liam Neeson }\tabularnewline
 & \textbf{\footnotesize{}DNC:}{\footnotesize{} Tom Gyllenhaal}\tabularnewline
 & \textbf{\footnotesize{}CVAE:}{\footnotesize{} Daryl and Aaron /{*}/
Carefully count Alfred Deniro /{*}/ Ponyo Joker possible}\tabularnewline
 & \textbf{\footnotesize{}VLSTM:}{\footnotesize{} Michael Bullock /{*}/
Michael /{*}/ Michael De}\tabularnewline
 & \textbf{\footnotesize{}VMED (K=3)}{\footnotesize{}: Edward or Leo
Dicaprio goes on /{*}/ Dicaprio will /{*}/ Dicaprio Tom has actually
in jack on road }\tabularnewline
\hline 
\multirow{6}{0.2\columnwidth}{\textbf{\footnotesize{}LJ user:}{\footnotesize{} Your very pretty
and your is very cute. May i add ya to my friends list ? }} & \textbf{\footnotesize{}Seq2Seq:}{\footnotesize{} I! Add and too you
back}\tabularnewline
 & \textbf{\footnotesize{}Seq2Seq-att: }{\footnotesize{}I did may! Away
back }\tabularnewline
 & \textbf{\footnotesize{}DNC: }{\footnotesize{}Sure you added}\tabularnewline
 & \textbf{\footnotesize{}CVAE:}{\footnotesize{} Sure ! I can add you
/{*}/ Thank cool of surely /{*}/ Thank hi of sure!}\tabularnewline
 & \textbf{\footnotesize{}VLSTM:}{\footnotesize{} Haha of do go /{*}/
Haha of sure! Yes lol /{*}/ Haha sure! Go}\tabularnewline
 & \textbf{\footnotesize{}VMED (K=3):}{\footnotesize{} Thank cool for
sure /{*}/ Add definitely! Sure surely friends /{*}/ Thank girl go
ahead}\tabularnewline
\hline 
\multirow{6}{0.2\columnwidth}{\textbf{\footnotesize{}Cornell movies: }{\footnotesize{}Nobody can
murder someone with their feelings . I mean that is crazy ?}} & \textbf{\footnotesize{}Seq2Seq:}{\footnotesize{} Yes you are}\tabularnewline
 & \textbf{\footnotesize{}Seq2Seq-att: }{\footnotesize{}Really not is
it}\tabularnewline
 & \textbf{\footnotesize{}DNC: }{\footnotesize{}Managing the truth}\tabularnewline
 & \textbf{\footnotesize{}CVAE: }{\footnotesize{}Sure not to be in here
/{*}/ Oh yes but i know! /{*}/ That to doing with here and}\tabularnewline
 & \textbf{\footnotesize{}VLSTM:}{\footnotesize{} I am dead . ! That
is ... /{*}/ I did not what way . /{*}/ I am not . But his things
with ...}\tabularnewline
 & \textbf{\footnotesize{}VMED (K=4):}{\footnotesize{} You are right!
/{*}/ That is crazy /{*}/ You can't know Jimmy }\tabularnewline
\hline 
\multirow{6}{0.2\columnwidth}{\textbf{\footnotesize{}OpenSubtitle:}{\footnotesize{} I'm obliged
for your hospitality. I appreciate it and your husband too.}} & \textbf{\footnotesize{}Seq2Seq: }{\footnotesize{}That is have got
coming about these }\tabularnewline
 & \textbf{\footnotesize{}Seq2Seq-att:}{\footnotesize{} May you not what
nothing nobody}\tabularnewline
 & \textbf{\footnotesize{}DNC: }{\footnotesize{}Yes i am taking this}\tabularnewline
 & \textbf{\footnotesize{}CVAE:}{\footnotesize{} No . that for good!
And okay /{*}/ All in the of two thing /{*}/ Sure. Is this! }\tabularnewline
 & \textbf{\footnotesize{}VLSTM:}{\footnotesize{} I ... /{*}/ I understand
/{*}/ I ! . ...}\tabularnewline
 & \textbf{\footnotesize{}VMED (K=3):}{\footnotesize{} I know. I can
afford /{*}/ I know nothing to store for you pass /{*}/ I know. Doing
anymore you father}\tabularnewline
\hline 
\end{tabular}
\end{table}

\textbf{Metric-based Analysi}s: We report results on four test datasets
in Table \ref{tab:BLEUs-and-A-Glove}. For BLEU scores, here we only
list results for BLEU-1 and 4. Other BLEUs show similar pattern and
will be listed in  Supplementary material. As clearly seen, VMED models
outperform other baselines over all metrics across four datasets.
In general, the performance of Seq2Seq is comparable with other deterministic
methods despite its simplicity. Surprisingly, CVAE or VLSTM does not
show much advantage over deterministic models. As we shall see, although
CVAE and VLSTM responses are diverse, they are often out of context.
Among different modes of VMED, there is often one best fit with the
data and thus shows superior performance. The optimal number of modes
in our experiments often falls to $K=3$, indicating that increasing
modes does not mean to improve accuracy. 

It should be noted that there is inconsistency between BLEU scores
and A-Glove metrics. This is because BLEU measures lexicon matching
while A-Glove evaluates semantic similarly in the embedding space.
For examples, two sentences having different words may share the same
meaning and lie close in the embedding space. In either case, compared
to others, our optimal VMED always achieves better performance.

\textbf{Qualitative Analysis}

Table \ref{tab:Examples-of-context-response} represents responses
generated by experimental models in reply to different input sentences.
The replies listed are chosen randomly from 50 generated responses
whose average of metric scores over all models are highest. For stochastic
models, we generate three times for each input, resulting in three
different responses. In general, the stochastic models often yield
longer and diverse sequences as expected. For closed-domain cases,
all models responses are fairly acceptable. Compared to the rest,
our VMED's responds seem to relate more to the context and contain
meaningful information. In this experiment, the open-domain input
seems nosier and harder than the closed-domain ones, thus create a
big challenge for all models. Despite that, the quality of VMED's
responses is superior to others. Among deterministic models, DNC's
generated responses look more reasonable than Seq2Seq's even though
its BLEU scores are not always higher. Perhaps, the reference to external
memory at every timestep enhances the coherence between output and
input, making the response more related to the context. VMED may inherit
this feature from its external memory and thus tends to produce reasonable
responses. By contrast, although responses from CVAE and VLSTM are
not trivial, they have more grammatical errors and sometimes unrelated
to the topic.

\section{Related Work}

With the recent revival of recurrent neural networks (RNNs), there
has been much effort spent on learning generative models of sequences.
Early attempts include training RNN to generate the next output given
previous sequence, demonstrating RNNs' ability to generate text and
handwriting images \cite{graves2013generating}. Later, encoder-decoder
architecture \cite{sutskever2014sequence} enables generating a whole
sequence in machine translation \cite{kalchbrenner2013recurrent},
text summation \cite{nallapati2016abstractive} and conversation generation
\cite{vinyals2015neural}. Although these models have achieved significant
empirical successes, they fall short to capture the complexity and
variability of sequential processes.

These limitations have recently triggered a considerable effort on
introducing variability into the encoder-decoder architecture. Most
of the methods focus on conditional VAE (CVAE) by constructing a variational
lower bound conditioned on the context. The setting can be found in
many applications including machine translation \cite{zhang2016variational}
and dialog generation \cite{bowman2016generating,serban2017hierarchical,shen2017conditional,zhao2017learning}.
A common trick is to place a neural net between the encoder and the
decoder to compute the Gaussian prior and posterior of the CVAE. This
design is further enhanced by the use of external memory \cite{chen2018hierarchical}
and reinforcement learning \cite{wen2017latent}. In contrast to this
design, our VMED uses recurrent latent variable approach \cite{chung2015recurrent},
that is, our model requires a CVAE for each step of generation. Besides,
our external memory is used for producing the latent distribution,
which is different from the one proposed in \cite{chen2018hierarchical}
where the memory is used only for holding long-term dependencies at
sentence level. Compared to variational addressing scheme mentioned
in \cite{bornschein2017variational}, our memory uses deterministic
addressing scheme, yet the memory content itself is used to introduce
randomness to the architecture. More relevant to our work is GTMM
\cite{gemici2017generative} where memory read-outs involve in constructing
the prior and posterior at every timesteps. However, this approach
uses Gaussian prior without conditional context.  

Using mixture of models instead of single Gaussian in VAE framework
is not a new concept. Works in \cite{dilokthanakul2016deep,jiangvariational}
and \cite{nalisnick2016approximate} proposed replacing the Gaussian
prior and posterior in VAE by MoGs for clustering and generating image
problems. Works in \cite{shu2016stochastic} and \cite{NIPS2017_7158}
applied MoG prior to model transitions between video frames and caption
generation, respectively. These methods use simple feed forward network
to produce Gaussian sub-distributions independently. In our model,
on the contrary, memory slots are strongly correlated with each others,
and thus modes in our MoG work together to define the shape of the
latent distributions at specific timestep. To the best of our knowledge,
our work is the first attempt to use an external memory to induce
mixture models  for sequence generation problems.

\section{Conclusions}

We propose a novel approach to sequence generation called Variational
Memory Encoder-Decoder (VMED) that introduces variability into encoder-decoder
architecture via the use of external memory as mixture model. By modeling
the latent temporal dependencies across timesteps, our VMED produces
a MoG representing the latent distribution. Each mode of the MoG associates
with some memory slot and thus captures some aspect of context supporting
generation process. To accommodate the MoG, we employ a $KL$ approximation
and we demonstrate that minimizing this approximation is equivalent
to minimizing the $KL$ divergence. We derive an upper bound on our
total timestep-wise $KL$ divergence and indicate that the optimization
of this upper bound is equivalent to fitting a continuous function
by an scaled MoG, which is in theory possible regardless of the function
form. This forms a theoretical basis for our model formulation using
MoG prior for every step of generation. We apply our proposed model
to conversation generation problem. The results demonstrate that VMED
outperforms recent advances both quantitatively and qualitatively.
Future explorations may involve implementing a dynamic number of modes
that enable learning of the optimal $K$ for each timestep. Another
aspect would be multi-person dialog setting, where our memory as mixture
model may be useful to capture more complex modes of speaking in the
dialog.

\bibliographystyle{plain}
\bibliography{}

\begin{thebibliography}{10}

\bibitem{bacharoglou2010approximation}
Athanassia Bacharoglou.
\newblock Approximation of probability distributions by convex mixtures of
  gaussian measures.
\newblock {\em Proceedings of the American Mathematical Society},
  138(7):2619--2628, 2010.

\bibitem{bahdanau2015neural}
Dzmitry Bahdanau, Kyunghyun Cho, and Yoshua Bengio.
\newblock Neural machine translation by jointly learning to align and
  translate.
\newblock {\em Proceedings of the International Conference on Learning
  Representations}, 2015.

\bibitem{bornschein2017variational}
J{\"o}rg Bornschein, Andriy Mnih, Daniel Zoran, and Danilo~Jimenez Rezende.
\newblock Variational memory addressing in generative models.
\newblock In {\em Advances in Neural Information Processing Systems}, pages
  3923--3932, 2017.

\bibitem{bowman2016generating}
Samuel~R Bowman, Luke Vilnis, Oriol Vinyals, Andrew Dai, Rafal Jozefowicz, and
  Samy Bengio.
\newblock Generating sentences from a continuous space.
\newblock In {\em Proceedings of The SIGNLL Conference on Computational Natural
  Language Learning}, pages 10--21, 2016.

\bibitem{britz2017efficient}
Denny Britz, Melody Guan, and Minh-Thang Luong.
\newblock Efficient attention using a fixed-size memory representation.
\newblock In {\em Proceedings of the Conference on Empirical Methods in Natural
  Language Processing}, pages 392--400, 2017.

\bibitem{chen2014systematic}
Boxing Chen and Colin Cherry.
\newblock A systematic comparison of smoothing techniques for sentence-level
  bleu.
\newblock In {\em Proceedings of the Ninth Workshop on Statistical Machine
  Translation}, pages 362--367, 2014.

\bibitem{chen2018hierarchical}
Hongshen Chen, Zhaochun Ren, Jiliang Tang, Yihong~Eric Zhao, and Dawei Yin.
\newblock Hierarchical variational memory network for dialogue generation.
\newblock In {\em Proceedings of the World Wide Web Conference on World Wide
  Web}, pages 1653--1662. International World Wide Web Conferences Steering
  Committee, 2018.

\bibitem{chung2015recurrent}
Junyoung Chung, Kyle Kastner, Laurent Dinh, Kratarth Goel, Aaron~C Courville,
  and Yoshua Bengio.
\newblock A recurrent latent variable model for sequential data.
\newblock In {\em Advances in Neural Information Processing Systems}, pages
  2980--2988, 2015.

\bibitem{dilokthanakul2016deep}
Nat Dilokthanakul, Pedro~AM Mediano, Marta Garnelo, Matthew~CH Lee, Hugh
  Salimbeni, Kai Arulkumaran, and Murray Shanahan.
\newblock Deep unsupervised clustering with gaussian mixture variational
  autoencoders.
\newblock {\em arXiv preprint arXiv:1611.02648}, 2016.

\bibitem{durrieu2012lower}
J-L Durrieu, J-Ph Thiran, and Finnian Kelly.
\newblock Lower and upper bounds for approximation of the kullback-leibler
  divergence between gaussian mixture models.
\newblock In {\em IEEE International Conference on Acoustics, Speech and Signal
  Processing.}, 2012.

\bibitem{forgues2014bootstrapping}
Gabriel Forgues, Joelle Pineau, Jean-Marie Larchev{\^e}que, and R{\'e}al
  Tremblay.
\newblock Bootstrapping dialog systems with word embeddings.
\newblock In {\em Nips, Modern Machine Learning and Natural Language Processing
  Workshop}, volume~2, 2014.

\bibitem{gemici2017generative}
Mevlana Gemici, Chia-Chun Hung, Adam Santoro, Greg Wayne, Shakir Mohamed,
  Danilo~J Rezende, David Amos, and Timothy Lillicrap.
\newblock Generative temporal models with memory.
\newblock {\em arXiv preprint arXiv:1702.04649}, 2017.

\bibitem{goodfellow2014generative}
Ian Goodfellow, Jean Pouget-Abadie, Mehdi Mirza, Bing Xu, David Warde-Farley,
  Sherjil Ozair, Aaron Courville, and Yoshua Bengio.
\newblock Generative adversarial nets.
\newblock In {\em Advances in Neural Information Processing Systems}, pages
  2672--2680, 2014.

\bibitem{graves2013generating}
Alex Graves.
\newblock Generating sequences with recurrent neural networks.
\newblock {\em arXiv preprint arXiv:1308.0850}, 2013.

\bibitem{graves2016hybrid}
Alex Graves, Greg Wayne, Malcolm Reynolds, Tim Harley, Ivo Danihelka, Agnieszka
  Grabska-Barwi{\'n}ska, Sergio~G{\'o}mez Colmenarejo, Edward Grefenstette,
  Tiago Ramalho, John Agapiou, et~al.
\newblock Hybrid computing using a neural network with dynamic external memory.
\newblock {\em Nature}, 538(7626):471--476, 2016.

\bibitem{hershey2007approximating}
John~R Hershey and Peder~A Olsen.
\newblock Approximating the kullback leibler divergence between gaussian
  mixture models.
\newblock In {\em IEEE International Conference on Acoustics, Speech and Signal
  Processing.}, 2007.

\bibitem{jiangvariational}
Zhuxi Jiang, Yin Zheng, Huachun Tan, Bangsheng Tang, and Hanning Zhou.
\newblock Variational deep embedding: An unsupervised and generative approach
  to clustering.
\newblock In {\em Proceedings of the International Joint Conference on
  Artificial Intelligence}, pages 1965--1972. International Joint Conference on
  Artificial Intelligence, 2017.

\bibitem{kalchbrenner2013recurrent}
Nal Kalchbrenner and Phil Blunsom.
\newblock Recurrent continuous translation models.
\newblock In {\em Proceedings of the Conference on Empirical Methods in Natural
  Language Processing}, pages 1700--1709, 2013.

\bibitem{kingma2014adam}
Diederik~P Kingma and Jimmy Ba.
\newblock Adam: A method for stochastic optimization.
\newblock {\em arXiv preprint arXiv:1412.6980}, 2014.

\bibitem{kingma2014semi}
Diederik~P Kingma, Shakir Mohamed, Danilo~Jimenez Rezende, and Max Welling.
\newblock Semi-supervised learning with deep generative models.
\newblock In {\em Advances in Neural Information Processing Systems}, pages
  3581--3589, 2014.

\bibitem{king14autoencoder}
Diederik~P Kingma and Max Welling.
\newblock Auto-encoding variational bayes.
\newblock In {\em Proceedings of the International Conference on Learning
  Representations}, 2014.

\bibitem{le2018dual_con}
Hung Le, Truyen Tran, and Svetha Venkatesh.
\newblock Dual control memory augmented neural networks for treatment
  recommendations.
\newblock In {\em Advances in Knowledge Discovery and Data Mining}, pages
  273--284, Cham, 2018. Springer International Publishing.

\bibitem{Le:2018:DMN:3219819.3219981}
Hung Le, Truyen Tran, and Svetha Venkatesh.
\newblock Dual memory neural computer for asynchronous two-view sequential
  learning.
\newblock In {\em Proceedings of the 24th ACM SIGKDD International Conference
  on Knowledge Discovery; Data Mining}, KDD '18, pages 1637--1645, New York,
  NY, USA, 2018. ACM.

\bibitem{levy2014neural}
Omer Levy and Yoav Goldberg.
\newblock Neural word embedding as implicit matrix factorization.
\newblock In {\em Advances in Neural Information Processing Systems}, pages
  2177--2185, 2014.

\bibitem{li2016diversity}
Jiwei Li, Michel Galley, Chris Brockett, Jianfeng Gao, and Bill Dolan.
\newblock A diversity-promoting objective function for neural conversation
  models.
\newblock In {\em Proceedings of the Conference of the North American Chapter
  of the Association for Computational Linguistics: Human Language
  Technologies}, pages 110--119, 2016.

\bibitem{lison2017not}
Pierre Lison and Serge Bibauw.
\newblock Not all dialogues are created equal: Instance weighting for neural
  conversational models.
\newblock In {\em Proceedings of the Annual SIGdial Meeting on Discourse and
  Dialogue}, pages 384--394, 2017.

\bibitem{maz1996approximate}
Vladimir Maz'ya and Gunther Schmidt.
\newblock On approximate approximations using gaussian kernels.
\newblock {\em IMA Journal of Numerical Analysis}, 16(1):13--29, 1996.

\bibitem{mikolov2013distributed}
Tomas Mikolov, Ilya Sutskever, Kai Chen, Greg~S Corrado, and Jeff Dean.
\newblock Distributed representations of words and phrases and their
  compositionality.
\newblock In {\em Advances in Neural Information Processing Systems}, pages
  3111--3119, 2013.

\bibitem{nalisnick2016approximate}
Eric Nalisnick, Lars Hertel, and Padhraic Smyth.
\newblock Approximate inference for deep latent gaussian mixtures.
\newblock In {\em NIPS Workshop on Bayesian Deep Learning}, volume~2, 2016.

\bibitem{nallapati2016abstractive}
Ramesh Nallapati, Bowen Zhou, Cicero dos Santos, Caglar Gulcehre, and Bing
  Xiang.
\newblock Abstractive text summarization using sequence-to-sequence rnns and
  beyond.
\newblock In {\em Proceedings of the SIGNLL Conference on Computational Natural
  Language Learning}, pages 280--290, 2016.

\bibitem{prakash2017condensed}
Aaditya Prakash, Siyuan Zhao, Sadid~A Hasan, Vivek~V Datla, Kathy Lee, Ashequl
  Qadir, Joey Liu, and Oladimeji Farri.
\newblock Condensed memory networks for clinical diagnostic inferencing.
\newblock In {\em Proceedings of the AAAI Conference on Artificial
  Intelligence}, pages 3274--3280, 2017.

\bibitem{rezende2014stochastic}
Danilo~Jimenez Rezende, Shakir Mohamed, and Daan Wierstra.
\newblock Stochastic backpropagation and approximate inference in deep
  generative models.
\newblock In {\em Proceedings of the International Conference on International
  Conference on Machine Learning}, pages II--1278. JMLR. org, 2014.

\bibitem{serban2017hierarchical}
Iulian~Vlad Serban, Alessandro Sordoni, Ryan Lowe, Laurent Charlin, Joelle
  Pineau, Aaron~C Courville, and Yoshua Bengio.
\newblock A hierarchical latent variable encoder-decoder model for generating
  dialogues.
\newblock In {\em Proceedings of the AAAI Conference on Artificial
  Intelligence}, pages 3295--3301, 2017.

\bibitem{shen2017conditional}
Xiaoyu Shen, Hui Su, Yanran Li, Wenjie Li, Shuzi Niu, Yang Zhao, Akiko Aizawa,
  and Guoping Long.
\newblock A conditional variational framework for dialog generation.
\newblock In {\em Proceedings of the Annual Meeting of the Association for
  Computational Linguistics (Volume 2: Short Papers)}, volume~2, pages
  504--509, 2017.

\bibitem{shu2016stochastic}
Rui Shu, James Brofos, Frank Zhang, Hung~Hai Bui, Mohammad Ghavamzadeh, and
  Mykel Kochenderfer.
\newblock Stochastic video prediction with conditional density estimation.
\newblock In {\em ECCV Workshop on Action and Anticipation for Visual
  Learning}, volume~2, 2016.

\bibitem{NIPS2015_5846}
Sainbayar Sukhbaatar, arthur szlam, Jason Weston, and Rob Fergus.
\newblock End-to-end memory networks.
\newblock In C.~Cortes, N.~D. Lawrence, D.~D. Lee, M.~Sugiyama, and R.~Garnett,
  editors, {\em Advances in Neural Information Processing Systems}, pages
  2440--2448. 2015.

\bibitem{sutskever2014sequence}
Ilya Sutskever, Oriol Vinyals, and Quoc~VV Le.
\newblock Sequence to sequence learning with neural networks.
\newblock In {\em Advances in Neural Information Processing Systems}, pages
  3104--3112, 2014.

\bibitem{vinyals2015neural}
Oriol Vinyals and Quoc Le.
\newblock A neural conversational model.
\newblock {\em arXiv preprint arXiv:1506.05869}, 2015.

\bibitem{NIPS2017_7158}
Liwei Wang, Alexander Schwing, and Svetlana Lazebnik.
\newblock Diverse and accurate image description using a variational
  auto-encoder with an additive gaussian encoding space.
\newblock In I.~Guyon, U.~V. Luxburg, S.~Bengio, H.~Wallach, R.~Fergus,
  S.~Vishwanathan, and R.~Garnett, editors, {\em Advances in Neural Information
  Processing Systems}, pages 5756--5766. 2017.

\bibitem{wang2016memory}
Mingxuan Wang, Zhengdong Lu, Hang Li, and Qun Liu.
\newblock Memory-enhanced decoder for neural machine translation.
\newblock In {\em Proceedings of the Conference on Empirical Methods in Natural
  Language Processing}, pages 278--286, 2016.

\bibitem{wen2017latent}
Tsung-Hsien Wen, Yishu Miao, Phil Blunsom, and Steve Young.
\newblock Latent intention dialogue models.
\newblock In {\em Proceedings of the International Conference on Machine
  Learning}, pages 3732--3741, 2017.

\bibitem{zhang2016variational}
Biao Zhang, Deyi Xiong, Hong Duan, Min Zhang, et~al.
\newblock Variational neural machine translation.
\newblock In {\em Proceedings of the Conference on Empirical Methods in Natural
  Language Processing}, pages 521--530, 2016.

\bibitem{zhao2017learning}
Tiancheng Zhao, Ran Zhao, and Maxine Eskenazi.
\newblock Learning discourse-level diversity for neural dialog models using
  conditional variational autoencoders.
\newblock In {\em Proceedings of the Annual Meeting of the Association for
  Computational Linguistics (Volume 1: Long Papers)}, volume~1, pages 654--664,
  2017.

\end{thebibliography}
\newpage{}

\renewcommand\thesubsection{\Alph{subsection}}

\section*{Supplementary material}

\subsection{Derivation of the Upper Bound on the $KL$ divergence}
\begin{thm}
The KL divergence between a Gaussian and a Mixture of Gaussians has
an upper bound $D_{var}$.
\end{thm}

\begin{proof}
$D_{var}$$\left(f\parallel g\right)$ \cite{hershey2007approximating}
is an approximation of $KL$ divergence between two Mixture of Gaussians
(MoG), which is defined as the following:

\begin{alignat}{1}
D_{var}\left(f\parallel g\right)= & \underset{j}{\sum}\pi_{j}^{f}\log\frac{\underset{j'}{\sum}\pi_{j'}^{f}e^{-KL\left(f_{j}\parallel f_{j'}\right)}}{\underset{i}{\sum}\pi_{i}^{g}e^{-KL\left(f_{j}\parallel g_{i}\right)}}\label{eq:ori_dvar}
\end{alignat}
In our case, $f$ is a Gaussian, a special case of MoG where the number
of mode equals one. Then, Eq. (\ref{eq:ori_dvar}) becomes:

\[
D_{var}\left(f\parallel g\right)=\log\frac{1}{\stackrel[i=1]{K}{\sum}\pi_{i}^{g}e^{-KL\left(f\parallel g^{i}\right)}}=-\log\stackrel[i=1]{K}{\sum}\pi^{i}e^{-KL\left(f\parallel g^{i}\right)}
\]
Let define the log-likelihood $L_{f}\left(g\right)=E_{f\left(x\right)}\left[\log g\left(x\right)\right]$,
the lower bound for $L_{f}\left(g\right)$ can be also be derived,
using variational parameters as follows:

\begin{alignat*}{1}
L_{f}\left(g\right)= & E_{f}\left[\log\left(\stackrel[i=1]{K}{\sum}\pi^{i}g^{i}\left(x\right)\right)\right]\\
= & \stackrel[-\infty]{+\infty}{\int}f\left(x\right)\log\left(\stackrel[i=1]{K}{\sum}\beta^{i}\pi^{i}\frac{g^{i}\left(x\right)}{\beta^{i}}\right)dx\\
\geq & \stackrel[i=1]{K}{\sum}\beta^{i}\stackrel[-\infty]{+\infty}{\int}f\left(x\right)\log\left(\pi^{i}\frac{g^{i}\left(x\right)}{\beta^{i}}\right)dx
\end{alignat*}
where $\beta^{i}\geq0$ and $\stackrel[i=1]{K}{\sum}\beta^{i}=1$.
According to \cite{durrieu2012lower}, maximizing the RHS of the above
inequality with respect to $\beta^{i}$ provides a lower bound for
$L_{f}\left(g\right)$:

\begin{alignat*}{1}
L_{f}\left(g\right)\geq & \log\stackrel[i=1]{K}{\sum}\pi^{i}e^{-KL\left(f\parallel g^{i}\right)}+L_{f}\left(f\right)\\
= & -D_{var}+L_{f}\left(f\right)\\
\Rightarrow D_{var}\geq & L_{f}\left(f\right)-L_{f}\left(g\right)\\
= & KL\left(f\parallel g\right)
\end{alignat*}
Therefore, the $KL$ divergence has an upper bound: $D_{var}$.
\end{proof}

\subsection{Derivation of the Upper Bound on the Total Timestep-wise $KL$ Divergence}
\begin{lem}
Chebyshev's sum inequality: \label{lem:Chebyshev's-sum-inequality:}\\
if 

\[
a_{1}\geq a_{2}\geq...\geq a_{n}
\]
and
\end{lem}

\[
b_{1}\geq b_{2}\geq...\geq b_{n}
\]
then\\
\[
\frac{1}{n}\stackrel[k=1]{n}{\sum}a_{k}b_{k}\geq\left(\frac{1}{n}\stackrel[k=1]{n}{\sum}a_{k}\right)\left(\frac{1}{n}\stackrel[k=1]{n}{\sum}b_{k}\right)
\]

\begin{proof}
Consider the sum:

\[
S=\stackrel[j=1]{n}{\sum}\stackrel[k=1]{n}{\sum}\left(a_{j}-a_{k}\right)\left(b_{j}-b_{k}\right)
\]
The two sequences are non-increasing, therefore $a_{j}-a_{k}$ and
$b_{j}-b_{k}$ have the same sign for any $j,k$. Hence $S\ge0$.
Opening the brackets, we deduce:

\[
{\displaystyle 0\leq2n\sum_{j=1}^{n}a_{j}b_{j}-2\sum_{j=1}^{n}a_{j}\,\sum_{k=1}^{n}b_{k}}
\]
whence:

\[
{\displaystyle {\frac{1}{n}}\sum_{j=1}^{n}a_{j}b_{j}\geq\left({\frac{1}{n}}\sum_{j=1}^{n}a_{j}\right)\,\left({\frac{1}{n}}\sum_{k=1}^{n}b_{k}\right)}
\]
\end{proof}
In our problem, $a_{i}=f_{i}\left(x\right)$ and $b_{i}=\log\left[g_{i}\left(x\right)\right]$,
$i=\overline{1,T}$. Under the assumption that at each step, thanks
to minimizing $D_{var}$, the approximation between the MoG and the
Gaussian is adequate to preserve the order of these values, that is,
if $f_{i}\left(x\right)\leq f_{j}\left(x\right)$, then $g_{i}\left(x\right)\leq g_{j}\left(x\right)$
and $\log\left[g_{i}\left(x\right)\right]\leq\log\left[g_{j}\left(x\right)\right]$.
Without loss of generality, we hypothesize that $f_{1}\left(x\right)\leq f_{2}\left(x\right)\leq...\leq f_{T}\left(x\right)$,
then we have $\log\left[g_{1}\left(x\right)\right]\leq\log\left[g_{2}\left(x\right)\right]\leq...\leq\log\left[g_{T}\left(x\right)\right]$.
Thus, applying Lemma \ref{lem:Chebyshev's-sum-inequality:}, we have:

\begin{alignat*}{1}
\frac{1}{T}\stackrel[t=1]{T}{\sum}f_{t}\left(x\right)\log\left[g_{t}\left(x\right)\right]dx\geq & \frac{1}{T}\stackrel[t=1]{T}{\sum}f_{t}\left(x\right)\frac{1}{T}\stackrel[t=1]{T}{\sum}\log\left[g_{t}\left(x\right)\right]dx\\
\Rightarrow\stackrel[-\infty]{+\infty}{\int}\stackrel[t=1]{T}{\sum}f_{t}\left(x\right)\log\left[g_{t}\left(x\right)\right]dx\geq & \stackrel[-\infty]{+\infty}{\int}\frac{1}{T}\stackrel[t=1]{T}{\sum}f_{t}\left(x\right)\stackrel[t=1]{T}{\sum}\log\left[g_{t}\left(x\right)\right]dx\\
\Rightarrow\stackrel[-\infty]{+\infty}{\int}\stackrel[t=1]{T}{\sum}f_{t}\left(x\right)\log\left[g_{t}\left(x\right)\right]dx\geq & \stackrel[-\infty]{+\infty}{\int}\frac{1}{T}\stackrel[t=1]{T}{\sum}f_{t}\left(x\right)\log\left[\stackrel[t=1]{T}{\prod}g_{t}\left(x\right)\right]dx
\end{alignat*}
Thus, the upper bound on the total timestep-wise $KL$ divergence
reads:

\[
\stackrel[-\infty]{+\infty}{\int}\stackrel[t=1]{T}{\sum}f_{t}\left(x\right)\log\left[f_{t}\left(x\right)\right]dx\quad-\stackrel[-\infty]{+\infty}{\int}\frac{1}{T}\stackrel[t=1]{T}{\sum}f_{t}\left(x\right)\log\left[\stackrel[t=1]{T}{\prod}g_{t}\left(x\right)\right]dx
\]

\subsection{Proof $\stackrel[t=1]{T}{\prod}g_{t}\left(x\right)=\stackrel[t=1]{T}{\prod}\stackrel[i=1]{K}{\sum}\pi_{t}^{i}g_{t}^{i}\left(x\right)$
is a Scaled MoG}
\begin{lem}
Product of two Gaussians is a scaled Gaussian. \label{lem:Product-of-two-1}
\end{lem}

\begin{proof}
Let $\mathcal{N}_{x}\left(\mu,\Sigma\right)$ denote a density of
$x$, then

\[
\mathcal{N}_{x}\left(\mu_{1},\Sigma_{1}\right)\cdot\mathcal{N}_{x}\left(\mu_{2},\Sigma_{2}\right)=c_{c}\mathcal{N}_{x}\left(\mu_{c},\Sigma_{c}\right)
\]
where:

\begin{alignat*}{1}
c_{c}= & \frac{1}{\sqrt{\det\left(2\pi\left(\Sigma_{1}+\Sigma_{2}\right)\right)}}\exp\left(-\frac{1}{2}\left(m_{1}-m_{2}\right)^{T}\left(\Sigma_{1}+\Sigma_{2}\right)^{-1}\left(m_{1}-m_{2}\right)\right)\\
m_{c}= & \left(\Sigma_{1}^{-1}+\Sigma_{2}^{-1}\right)^{-1}\left(\Sigma_{1}^{-1}m_{1}+\Sigma_{2}^{-1}m_{2}\right)\\
\Sigma_{c}= & \left(\Sigma_{1}^{-1}+\Sigma_{2}^{-1}\right)
\end{alignat*}
\end{proof}
\begin{lem}
Product of two MoGs is proportional to an MoG. \label{lem:Product-of-two}
\end{lem}

\begin{proof}
Let $g_{1}\left(x\right)=\stackrel[i=1]{K_{1}}{\sum}\pi_{1,i}\mathcal{N}_{x}\left(\mu_{1,i},\Sigma_{1,i}\right)$
and $g_{2}\left(x\right)=\stackrel[j=1]{K_{2}}{\sum}\pi_{2,j}\mathcal{N}_{x}\left(\mu_{2,j},\Sigma_{2,j}\right)$
are two Mixtures of Gaussians. We have:

\begin{align}
g_{1}\left(x\right)\cdot g_{2}\left(x\right)= & \stackrel[i=1]{K_{1}}{\sum}\pi_{1,i}\mathcal{N}_{x}\left(\mu_{1,i},\Sigma_{1,i}\right)\cdot\stackrel[j=1]{K_{2}}{\sum}\pi_{2,j}\mathcal{N}_{x}\left(\mu_{2,j},\Sigma_{2,j}\right)\nonumber \\
= & \stackrel[i=1]{K_{1}}{\sum}\stackrel[,j=1]{K_{2}}{\sum}\pi_{1,i}\pi_{2,j}\mathcal{N}_{x}\left(\mu_{1,i},\Sigma_{1,i}\right)\cdot\mathcal{N}_{x}\left(\mu_{2,j},\Sigma_{2,j}\right)\label{eq:big_mix}
\end{align}
By applying Lemma \ref{lem:Product-of-two-1} to Eq. (\ref{eq:big_mix}),
we have

\begin{alignat}{1}
g_{1}\left(x\right)\cdot g_{2}\left(x\right)= & \stackrel[i=1]{K_{1}}{\sum}\stackrel[,j=1]{K_{2}}{\sum}\pi_{1,i}\pi_{2,j}c_{ij}\mathcal{N}_{x}\left(\mu_{ij},\Sigma_{ij}\right)\nonumber \\
= & \:C\stackrel[i=1]{K_{1}}{\sum}\stackrel[,j=1]{K_{2}}{\sum}\frac{\pi_{1,i}\pi_{2,j}c_{ij}}{C}\mathcal{N}_{x}\left(\mu_{ij},\Sigma_{ij}\right)\label{eq:mixprop}
\end{alignat}
where $C=\stackrel[i=1]{K_{1}}{\sum}\stackrel[,j=1]{K_{2}}{\sum}\pi_{1,i}\pi_{2,j}c_{ij}$.
Clearly, Eq. (\ref{eq:mixprop}) is proportional to an MoG with $K_{1}\cdot K_{2}$
modes 
\end{proof}
\begin{thm}
$\stackrel[t=1]{T}{\prod}g_{t}\left(x\right)=\stackrel[t=1]{T}{\prod}\stackrel[i=1]{K}{\sum}\pi_{t}^{i}g_{t}^{i}\left(x\right)$
is a scaled MoG.
\end{thm}

\begin{proof}
By induction from Lemma \ref{lem:Product-of-two}, we can easily show
that product of $T$ MoGs is also proportional to an MoG. That means
$\stackrel[t=1]{T}{\prod}g_{t}\left(x\right)$ equals to a scaled
MoG.
\end{proof}

\subsection{Details of Data Descriptions and Model Implementations}

Here we list all datasets used in our experiments:
\begin{itemize}
\item Open-domain datasets: 
\begin{itemize}
\item Cornell movie dialog: This corpus contains a large metadata-rich collection
of fictional conversations extracted from 617 raw movies with 220,579
conversational exchanges between 10,292 pairs of movie characters.
For each dialog, we preprocess the data by limiting the context length
and the utterance output length to 20 and 10, respectively. The vocabulary
is kept to top 20,000 frequently-used words in the dataset.
\item OpenSubtitles: This dataset consists of movie conversations in XML
format. It also contains sentences uttered by characters in movies,
yet it is much bigger and noisier than Cornell dataset. After preprocessing
as above, there are more than 1.6 million pairs of contexts and utterance
with chosen vocabulary of 40,000 words. 
\end{itemize}
\item Closed-domain datasets:: 
\begin{itemize}
\item Live Journal (LJ) user question-answering dataset: question-answer
dialog by LJ users who are members of anxiety, arthritis, asthma,
autism, depression, diabetes, and obesity LJ communities\footnote{\url{https://www.livejournal.com/}}.
After preprocessing as above, we get a dataset of more than 112,000
conversations. We limit the vocabulary size to 20,000 most common
words. 
\item Reddit comments dataset: This dataset consists of posts and comments
about movies in Reddit website\footnote{\url{https://www.reddit.com/r/movies/}}.
A single post may have multiple comments constituting a multi-people
dialog amongst the poster and commentors, which makes this dataset
the most challenging one. We crawl over four millions posts from Reddit
website and after preprocessing by retaining conversations whose utterance's
length are less than 20, we have a dataset of nearly 200 thousand
conversations with a vocabulary of more than 16 thousand words. 
\end{itemize}
\end{itemize}
We trained with the following hyperparameters (according to the performance
on the validate dataset): word embedding has size 96 and is shared
across everywhere. We initialize the word embedding from Google's
Word2Vec \cite{mikolov2013distributed} pretrained word vectors. The
hidden dimension of LSTM in all controllers is set to 768 for all
datasets except the big OpenSubtitles whose LSTM dimension is 1024.
The number of LSTM layers for every controllers is set to 3. All the
initial weights are sampled from a normal distribution with mean $0$,
standard deviation 0.$1$. The mini-batch size is chosen as 256. The
models are trained end-to-end using the Adam optimizer \cite{kingma2014adam}
with a learning rate of 0.001 and gradient clipping at 10. For models
using memory, we set the number and the size of memory slots to 16
and 64, respectively. As indicated in \cite{bowman2016generating},
it is not trivial to optimize VAE with RNN-like decoder due to the
vanishing latent variable problem. Hence, to make the variational
models in our experiments converge we have to use the $KL$ annealing
trick by adding to the $KL$ loss term an annealing coefficient $\alpha$
starts with a very small value and gradually increase up to 1.

\subsection{Full Reports on Model Performance}

\begin{table}[H]
\begin{centering}
\begin{tabular}{|c|c|c|c|c||c|}
\hline 
Model & BLEU-1 & BLEU-2 & BLEU-3 & BLEU-4 & A-glove\tabularnewline
\hline 
\hline 
Seq2Seq & 18.4 & 14.5 & 12.1 & 9.5 & 0.52\tabularnewline
\hline 
Seq2Seq-att & 17.7 & 14.0 & 11.7 & 9.2 & 0.54\tabularnewline
\hline 
DNC & 17.6 & 13.9 & 11.5 & 9.0 & 0.51\tabularnewline
\hline 
CVAE & 16.5 & 13.0 & 10.9 & 8.5 & 0.56\tabularnewline
\hline 
VLSTM & 18.6 & 14.8 & 12.4 & 9.7 & 0.59\tabularnewline
\hline 
VMED (K=1) & 20.7 & 16.5 & 13.8 & 10.8 & 0.57\tabularnewline
\hline 
VMED (K=2) & 22.3 & 18.0 & 15.2 & 11.9 & \textbf{0.64}\tabularnewline
\hline 
VMED (K=3) & 19.4 & 15.6 & 13.2 & 10.4 & 0.63\tabularnewline
\hline 
VMED (K=4) & \textbf{23.1} & \textbf{18.5} & \textbf{15.5} & \textbf{12.3} & 0.61\tabularnewline
\hline 
\end{tabular}
\par\end{centering}
\caption{Results on Cornell Movies}
\end{table}

\begin{table}[H]
\begin{centering}
\begin{tabular}{|c|c|c|c|c||c|}
\hline 
Model & BLEU-1 & BLEU-2 & BLEU-3 & BLEU-4 & A-glove\tabularnewline
\hline 
\hline 
Seq2Seq & 11.4 & 8.7 & 7.1 & 5.4 & 0.29\tabularnewline
\hline 
Seq2Seq-att & 13.2 & 10.2 & 8.4 & 6.5 & 0.42\tabularnewline
\hline 
DNC & 14.3 & 11.2 & 9.3 & 7.2 & 0.47\tabularnewline
\hline 
CVAE & 13.5 & 10.2 & 8.4 & 6.6 & 0.45\tabularnewline
\hline 
VLSTM & 16.4 & 12.7 & 10.4 & 8.1 & 0.43\tabularnewline
\hline 
VMED (K=1) & 12.9 & 9.5 & 7.5 & 6.2 & 0.44\tabularnewline
\hline 
VMED (K=2) & 15.3 & 13.8 & 10.4 & 8.8 & 0.49\tabularnewline
\hline 
VMED (K=3) & \textbf{24.8} & \textbf{19.7} & \textbf{16.4} & \textbf{12.9} & \textbf{0.54}\tabularnewline
\hline 
VMED (K=4) & 17.9 & 14.2 & 11.8 & 9.3 & 0.52\tabularnewline
\hline 
\end{tabular}
\par\end{centering}
\caption{Results on OpenSubtitles}
\end{table}

\begin{table}[H]
\begin{centering}
\begin{tabular}{|c|c|c|c|c||c|}
\hline 
Model & BLEU-1 & BLEU-2 & BLEU-3 & BLEU-4 & A-glove\tabularnewline
\hline 
\hline 
Seq2Seq & 13.1 & 10.1 & 8.3 & 6.4 & 0.45\tabularnewline
\hline 
Seq2Seq-att & 11.4 & 8.7 & 7.1 & 5.6 & 0.49\tabularnewline
\hline 
DNC & 12.4 & 9.6 & 7.8 & 6.1 & 0.47\tabularnewline
\hline 
CVAE & 12.2 & 9.4 & 7.7 & 6.0 & 0.48\tabularnewline
\hline 
VLSTM & 11.5 & 8.8 & 7.3 & 5.6 & 0.46\tabularnewline
\hline 
VMED (K=1) & 13.7 & 10.7 & 8.9 & 6.9 & 0.47\tabularnewline
\hline 
VMED (K=2) & 15.4 & 12.2 & 10.1 & 7.9 & \textbf{0.51}\tabularnewline
\hline 
VMED (K=3) & \textbf{18.1} & \textbf{14.8} & \textbf{12.4} & \textbf{9.8} & 0.49\tabularnewline
\hline 
VMED (K=4) & 14.4 & 11.4 & 9.5 & 7.5 & 0.47\tabularnewline
\hline 
\end{tabular}
\par\end{centering}
\caption{Results on LJ users question-answering}
\end{table}

\begin{table}[H]
\begin{centering}
\begin{tabular}{|c|c|c|c|c||c|}
\hline 
Model & BLEU-1 & BLEU-2 & BLEU-3 & BLEU-4 & A-glove\tabularnewline
\hline 
\hline 
Seq2Seq & 7.5 & 5.5 & 4.4 & 3.3 & 0.31\tabularnewline
\hline 
Seq2Seq-att & 5.5 & 4.0 & 3.1 & 2.4 & 0.25\tabularnewline
\hline 
DNC & 7.5 & 5.6 & 4.5 & 3.4 & 0.28\tabularnewline
\hline 
CVAE & 5.3 & 4.3 & 3.6 & 2.8 & 0.39\tabularnewline
\hline 
VLSTM & 6.9 & 5.1 & 4.1 & 3.1 & 0.27\tabularnewline
\hline 
VMED (K=1) & 9.1 & 6.8 & 5.5 & 4.3 & 0.39\tabularnewline
\hline 
VMED (K=2) & 9.2 & 7.0 & 5.7 & 4.4 & 0.38\tabularnewline
\hline 
VMED (K=3) & \textbf{12.3} & \textbf{9.7} & \textbf{8.1} & \textbf{6.4} & \textbf{0.46}\tabularnewline
\hline 
VMED (K=4) & 8.6 & 6.9 & 5.9 & 4.6 & 0.41\tabularnewline
\hline 
\end{tabular}
\par\end{centering}
\caption{Results on Reddit comments}
\end{table}

\end{document}